\documentclass[letterpaper, 10 pt, conference]{ieeeconf}  

\IEEEoverridecommandlockouts                              
\usepackage{color}
\usepackage{amssymb}  
\usepackage{booktabs}
\usepackage{tabularx}
\usepackage{arydshln}
\newlength\subtabdist
\usepackage[utf8]{inputenc}
\usepackage{graphicx}
\usepackage{xcolor}
\usepackage{MnSymbol}
\usepackage[linesnumbered,ruled]{algorithm2e}
\usepackage{tikz}
\usepackage[utf8]{inputenc}
\newtheorem{properties}{Property}
\newtheorem{proposition}{Proposition}
\newtheorem{problem}{Problem}
\newtheorem{definition}{Definition}
\newtheorem{remark}{Remark}
\newtheorem{theorem}{Theorem}
\usepackage{pgfplots} 
\pgfplotsset{compat=newest} 
\pgfplotsset{plot coordinates/math parser=false} 
\newlength\figureheight 
\newlength\figurewidth 
\usepackage[strict=true,style=english]{csquotes}

\newcommand{\conj}{\bigtriangleup}
\newcommand{\disj}{\bigtriangledown}
\newcommand{\nfnc}{n}
\newcommand{\impl}{\rhd}
\newcommand{\argmin}{\mathop{\mathrm{argmin}}} 

\newcommand{\argmax}{\mathop{\mathrm{argmax}}} 

\title{\LARGE \bf{
Arithmetic-Geometric Mean Robustness for
Control from Signal Temporal Logic Specifications
}}

\author{*Noushin Mehdipour$^{1}$, *Cristian-Ioan Vasile$^{2}$ and Calin Belta$^{1}$
\thanks{*These authors contributed equally. This work was partially supported at Boston University by the National Science Foundation under grants IIS-1723995, CPS-1446151, and CMMI-1400167}
\thanks{${}^1$ Noushin Mehdipour (noushinm@bu.edu), Calin Belta (cbelta@bu.edu) are with the Division of Systems Engineering at Boston University, Boston, MA, USA. ${}^2$ Cristian-Ioan Vasile (cvasile@mit.edu) is with the Laboratory for Information and Decision Systems (LIDS) at Massachusetts Institute of Technology, Cambridge, MA, USA.}%
}
\begin{document}
\maketitle
\thispagestyle{empty}
\pagestyle{empty}

\begin{abstract}
We present a new average-based robustness score for Signal Temporal Logic (STL) and a framework for optimal control of a dynamical system under STL constraints. By averaging the scores of different specifications or subformulae at different time points, our new definition highlights the frequency of satisfaction, as well as how robustly each specification is satisfied at each time point. We show that this definition provides a better score for how well a specification is satisfied. Its usefulness in monitoring and control synthesis problems is illustrated through case studies.
\end{abstract}
\section{INTRODUCTION}
\label{intro}
Formal methods have been recently used to express system behavior under complex temporal requirements, verify whether the system execution meets the desired requirements, or control the system to satisfy desirable specifications \cite{belta}. Temporal Logics including Linear Temporal Logics (LTL) \cite{ltl}, Metric Temporal Logic (MTL) \cite{mtl}, Signal Temporal Logic (STL) \cite{stl} and Time Window Temporal Logic (TWTL) \cite{twtl} allow precise description of system properties over time. STL is equipped with qualitative and quantitative semantics, meaning that it not only can assess whether the system execution meets the desired requirements but also provides a measure of how well requirements are met, also known as robustness. As a result, STL has been widely used for many control purposes including path planning and motion planning \cite{IROS}, \cite{shoukry} or synthesis problems \cite{dorsa}. Higher robustness score shows a stronger satisfaction of the desired specifications. Therefore, it is desirable to maximize the robustness score in order to improve system behavior to satisfy desired temporal specifications.

The traditional robustness score introduced in \cite{donze} is non-convex and non-differentiable; therefore, 
it is not possible to use powerful optimization techniques to maximize it. Previous works for control under STL constraints focused on using heuristic algorithms or encoding constraints as Mixed Integer Linear Programming (MILP). Heuristic optimization approaches such as Particle Swarm Optimization, Simulated Annealing and Rapidly Exploring Random Trees (RRTs) were used for synthesis, falsification and control problems \cite{cdc}, \cite{SA}, \cite{rrt}. Heuristic approaches do not require a smooth objective function; however, these algorithms do not always provide a guarantee to find the optima and have many user-defined parameters that need to be set in advance. 

Encoding temporal logic specifications as linear and boolean constraints was studied in \cite{raman}, \cite{milp} and MILP optimization solvers such as Gurobi were used to solve the control synthesis problem. The most critical issue with MILPs is that they do not scale well as the number of variables increases, resulting in a NP-complete problem. For instance, to encode the temporal operator $eventually$ as MILP constraints, we need to add integer (binary) variables for each time point in the specified interval. Therefore, this approach could fail when solving problems with many variables or complex temporal constraints. Moreover, MILP implementations require all constraints (including the system dynamics in the control problem) to be linear. As a result, nonlinear dynamics must be linearized, if linearizable, which involves approximation. 

Recently, there have been efforts to smooth the robustness function (score) in order to use gradient-based optimization algorithms. In \cite{husam}, \cite{li}, the authors used smooth approximations of maximum and minimum functions to define a smooth robustness score in order to solve a control problem. Even though these works solved the non-differentiability issue, the resulting smooth approximation had errors compared to the traditional robustness. Therefore, positive robustness did not necessarily mean satisfaction of the specification unless it was greater than a pre-defined threshold.

The main drawback of these works is that traditional robustness is defined by the most critical point (most satisfaction or most violation). In \cite{akazaki}, authors defined average STL robustness for continuous-time signals and defined positive and negative robustness to solve a falsification problem. Authors in \cite{filter} described MTL as linear time-invariant filters and used the average robustness for monitoring purposes. \cite{discrete} improved robustness for discrete signals by defining Discrete Average Space Robustness, and removed its nonsmoothness by approximating to a simplified version. These works refined robustness score only for temporal operators while using traditional maximum and minimum functions for other operators.

Our main contribution of this paper is proposing a new average-based robustness score, which we call Arithmetic-Geometric Mean (AGM) robustness. This new quantitative semantics uses arithmetic and geometric means to take into account the robustness degrees for all the subformulae and at every time point in the horizon, and not just the most satisfying or violating ones. As a result, our robustness definition rewards policies that satisfy the requirements at more time steps and with higher scores. We show that this novel robustness definition provides a better margin in which the specification is still satisfied when external disturbances or system perturbations exist. Moreover, our normalized signed robustness degree 
provides a meaningful comparison when specifications involve requirements over signals with different scales. The advantages of our new definition in both monitoring and control problems are illustrated through case studies through the paper. We compare our results with those obtained using MILPs and smooth approximation methods.
\section{PRELIMINARIES}
\label{preliminaries}
Let $f:\mathbb{R}^n \rightarrow \mathbb{R}$ be a real function. We define $[f]_ + = {\small \begin{cases} f & f > 0\\ 0 & \text{otherwise}\end{cases}}$ and $[f]_- = - [ - f]_+$, where $f =[f]_+ + [f]_-$.
\subsection{Signal Temporal Logics (STL)}
STL was introduced in \cite{stl} to monitor temporal properties of real-valued signals. Consider a discrete time sequence $\tau:=\{t_k |k\in\mathbb{Z}_{\geq 0}\}$. A \textit{signal} $S$ is a function $S:\tau\rightarrow\mathbb{R}^n$ that maps each time point $t_k \in \tau$ to an $n$-dimensional vector of real values $S[t_k]$, with $s_i$ being its $i$th component. Assume $[a,b]$ is the set of all $t_k \in \tau$ starting from $a$ up to $b$, with $a,b \in \tau; \; b > a\geq 0$. STL Syntax is defined as:
\begin{equation}
\label{eq:syntax}
\varphi:=\top \mid \mu \mid \lnot\varphi \mid \varphi_1\land\varphi_2 \mid \varphi_1\mathbf{U}_{[a,b]} \varphi_2,
\end{equation}
where $\top$ is the logical \textit{True}, $\mu$ is a \textit{predicate}, $\lnot$, $\land$ are the Boolean \textit{negation} and \textit{conjunction} operators, respectively, and $\mathbf{U}$ is the temporal \textit{until} operator.
Logical {\em False} is $\bot := \lnot \top$.
Other Boolean and temporal operators are defined as $\varphi_1\lor\varphi_2:=\lnot(\lnot\varphi_1\land\lnot\varphi_2)$, $\mathbf{F}_{[a,b]}\varphi:=\top\mathbf{U}_{[a,b]}\varphi$, $\mathbf{G}_{[a,b]}\varphi:=\lnot\mathbf{F}_{[a,b]}\lnot\varphi$.
In this paper, we focus on $\mathbf{F}$ and $\mathbf{G}$ operators, 
rather than $\mathbf{U}$. The temporal operator \textit{Finally} or \textit{eventually} ($\mathbf{F}_{[a,b]}\varphi$) states that \enquote{at some time point in $[a,b]$ the specification $\varphi$ must be True}; while \textit{globally} or \textit{always} ($\mathbf{G}_{[a,b]}\varphi$) states that \enquote{$\varphi$ must be True at all times in $[a,b]$}. The \textit{until} operator ($\varphi_1 \mathbf{U}_{[a,b]}\varphi_2$) states that \enquote{$\varphi_2$ must become True at some time point within $[a,b]$ and $\varphi_1$ must be always True prior to that}. A specification written in STL consists of predicates $\mu:=l(S)\geq 0$, where $l: \mathbb{R}^n \to \mathbb{R}$ is a real, possibly nonlinear, function defined over values of elements of $S$ (for instance, $s_1^3+s_2^2-2 \geq 0$ or $-2s_2+10 \geq 0$) connected by Boolean 
and temporal operators. 

The STL qualitative semantics shows \textit{whether} a signal $S$ satisfies a given specification $\varphi$ at time $t$, i.e., $S[t] \models \varphi$ or not, i.e., $S[t] \nmodels \varphi$, and its quantitative semantics, also known as \textit{robustness}, measures \textit{how much} the signal is satisfying or violating the specification. 
\begin{definition}[STL Robustness] 
Given a specification $\varphi$ and a signal $S$, the robustness score
$\rho(\varphi,S,t)$ at time $t$ is recursively computed as \cite{donze}: 
\begin{equation}
\label{eq:org}
\begin{aligned}
\rho ( {\top ,S,t} ) &: =\rho _{\top} ,\\
\rho ( {\bot ,S,t} ) &: =-\rho _{\top},\\
\rho ( {\mu ,S,t} ) &: =  l(S[t])),\\
\rho \left( {\neg \varphi ,S,t} \right) &: =  - \rho (\varphi ,S,t),\\
\rho \left( {\varphi_1  \wedge \varphi_2 ,S,t} \right) &:= \min \left( {\rho (\varphi_1 ,S,t),\rho (\varphi_2 ,S,t)} \right),\\
\rho \left( {\varphi_1  \vee \varphi_2 ,S,t} \right) &:= \max \left( {\rho (\varphi_1 ,S,t),\rho (\varphi_2 ,S,t)} \right),\\
\rho \left( {{\mathbf{G}_{[a,b]}}\varphi ,S,t} \right) &: = \mathop {\min }\limits_{t_k' \in{[t + a,t + b]}} {\rho (\varphi ,S,t_k')},\\
\rho \left( {{\mathbf{F}_{[a,b]}}\varphi ,S,t} \right) &: = \mathop {\max }\limits_{t_k' \in{[t + a,t + b]}} {\rho (\varphi ,S,t_k')},
\end{aligned}
\end{equation}
where $\rho _{\top} \in \mathbb{R} \cup \{+\infty \}$ is the maximum robustness.
\end{definition}
\begin{theorem}
The robustness score is sound, meaning that $\rho \left( {\varphi ,S,t} \right) > 0$ implies that signal $S$ satisfies $\varphi$ at time $t$, and $\rho \left( {\varphi ,S,t} \right) < 0$ implies that $S$ violates $\varphi$ at time $t$.
\end{theorem}
We denote the robustness score of specification $\varphi$ at time $0$ with respect to the signal $S$ by $\rho(\varphi,S)$. We refer to this definition as traditional robustness score.
\subsection{$\min$/$\max$ Approximation}
The $\min$ and $\max$ functions in the robustness definition in \eqref{eq:org} result in a non-differentiable robustness score. This non-differentiability can be removed by replacing $\max$ and $\min$ functions with the following smooth approximations: 
\begin{equation}
\label{eq:soft}
\begin{array}{l}
{\max_\beta}(a_1,\ldots,a_m) \approx \frac{1}{\beta}\ln\sum_{i=1}^m e^{\beta a_i}, \\
{\min_\beta}(a_1,\ldots,a_m) \approx - \frac{1}{\beta}\ln\sum_{i=1}^m e^{-\beta a_i}.
\end{array}
\end{equation}
For different values of $\beta$ different robustness scores are found, resulting in an error. It is shown in \cite{li} that the approximation error approaches $0$ as $\beta$ goes to $\infty$. We refer to this definition as approximation robustness score $\tilde{\rho}$.
\section{PROBLEM STATEMENT}
Consider a discrete-time dynamical system given by:
\begin{equation}
\label{eq:dynamics}
\begin{array}{l}
q[k+1]=f(q[k],u[k]),\\
q[0]=q_0,
\end{array}
\end{equation}
where $q[k] \in \mathbf{Q} \subseteq\mathbb{R}^n$ is the state of the system and $u[k] \in \mathbf{U} \subseteq\mathbb{R}^m$ is the control input at the $k$th time step $k \in \mathbb{Z}_{\geq 0}$; $q_0 \in \mathbf{Q} $ is the initial state and $f$ is a 
function representing the dynamics of the system. Given the initial state $q_0$ and control sequence $u=\{u[0]u[1]...\}$, system trajectory $q=\{q[0]q[1]q[2]...\}$ is generated using \eqref{eq:dynamics}; which we denote by $\langle q,u \rangle$. Consider a 
cost function $J(u[k],q[k+1])$ representing the cost of applying the control input $u[k]$. Assume system temporal requirements are given by a STL formula $\phi$ with a time horizon $T$, which is the largest time step for which signal values are needed in order to compute the robustness for the current time point. 
The control synthesis problem can be formulated as determining a control policy $u^*=\{u^*[0]u^*[1]...u^*[T-1]\}$ such that the system trajectory satisfies the STL specification $\phi$ while optimizing the cost:
\begin{equation}
\label{eq: cost}
\begin{array}{c}
u^*=\argmin _u \sum\limits_{k=0}^{T-1} J(u[k],q[k+1]),\\
\text{s.t.}\;\;\;\;\; \langle q,{u}\rangle\models\phi.\\
\end{array}
\end{equation}
As stated in the Sec.~\ref{intro}, previous works used heuristic algorithms, MILP encoding, and gradient ascent to solve \eqref{eq: cost} and found control policies to generate trajectories that satisfy STL constraints with the traditional and smooth robustness definition. The main shortcoming of the traditional robustness score is that it only considers the robustness of the most satisfying or violating part of the specification without taking into account satisfaction of the other parts. We address this limitation by defining a new version of robustness. 
\section{ARITHMETIC-GEOMETRIC MEAN (AGM) ROBUSTNESS}
We define a novel robustness score $\eta$ based on arithmetic and geometric means
instead of the $\max$ and $\min$ functions in the traditional definition.
We show that our normalized signed robustness score $\eta \in [-1,1]$
provides a better understanding of system properties, 
where $\eta \in (0,1]$ corresponds to satisfaction of the specification,
$\eta \in [-1,0)$ shows violation, and $\eta =0$ indicates inconclusiveness.
Moreover, $| \eta |$  is a measure of how well the specification is satisfied or violated. 

Consider a discrete time series $\tau:=\{t_k |k\in\mathbb{Z}_{\geq 0}\}$. Signal $S$ is a function $S:\tau\rightarrow\mathbb{R}^n$ that maps each time point $t_k \in \tau$ to an $n$-dimensional vector of real values $S[t_k]$, with $s_i$ being its $i$th element. 
Throughout the definitions and proofs, we assume that we have bounded signals, and all their components are normalized to the interval $[-1, 1]$.
\begin{definition}[AGM Robustness]
Let $S: \tau \to {[-1,1]}^n$ and $\varphi: s_i - \pi \ge 0$ where $\pi \in [-1,1]$
. The normalized signed AGM robustness $\eta (\varphi,S,t)$ 
with respect to the signal $S$ at time $t$ is defined as:
\begin{equation}
\label{eq:agm-def-base-cases}
\begin{aligned}
\eta (\top,S ,t) &: = 1\\
\eta (\bot,S ,t) &: = -1\\
\eta (\varphi,S ,t) &: =\frac{1}{2}( s_i[t]  - {\pi }),\\
\eta (\lnot \varphi ,S ,t) &: =  - \eta(\varphi ,S ,t).
\end{aligned}
\end{equation}
For combination of other boolean and temporal operators in a time interval $[a,b]$, AGM robustness is recursively defined using \eqref{eq:agm-sat} and \eqref{eq:agm-unsat}; with $[a,b]=\{t_k | t_k,a,b \in \tau;a \leq t_k \leq b; \; \; b > a\geq 0 \}$ and $N$ being the number of time points in $[a,b]$.
\begin{figure*}[htb]
\begin{equation}
\label{eq:agm-sat}
\begin{array}{l}
\eta ({\varphi _1} \land ... \land {\varphi _m},S,t\mid \forall i \in [1,...,m]\;.\;\eta ({\varphi _i},S,t) > 0): = \sqrt[m]{{\prod\limits_{i = 1,...,m} {\left( {1 + \eta ({\varphi _i},S,t)} \right)} }} - 1 \\ 
\\
\eta ({\varphi _1} \lor ... \lor{\varphi _m},S ,t \mid \exists i \in [1,...,m] \ .\ \eta ({\varphi}_i,S,t)>0) : = \frac{1}{m}\sum\limits_{i=1,...,m} [\eta(\varphi_i, S, t)]_+ \\
\\
\eta ({{\mathbf{G}}_{[a,b]}}\varphi ,S ,t \mid \forall t_k' \in [t+a, t+b] \ .\ \eta ({\varphi},S,t_k')>0) : = \sqrt[N]{ {\prod\limits_{t_k' \in [t+a,t+b] } {\left( {1 + \eta (\varphi, S, t_k')} \right)} } } - 1\\
\\
\eta ({{\mathbf{F}}_{[a,b]}}\varphi ,S ,t \mid \exists t_k' \in [t+a, t+b] \ .\ \eta ({\varphi},S,t_k')>0 ) : = \frac{1}{N}\sum\limits_{t_k' \in[t+a,t+b] } [\eta(\varphi, S, t_k')]_+ 
\end{array}
\end{equation} \vspace{-10pt}
\end{figure*}
\begin{figure*}[htb]
\begin{equation}
\label{eq:agm-unsat}
\begin{array}{l}
\eta ({\varphi _1} \land ... \land {\varphi _m},S,t\mid \exists i \in [1,...,m] \ .\ \eta ({\varphi}_i,S,t) \le 0)
:=\frac{1}{m}\sum\limits_{i=1,...,m} [\eta(\varphi_i, S, t)]_- \\
\\
\eta (\varphi_1 \lor ...\lor \varphi_m,S,t \mid \forall i \in [1,...,m]\;.\;\eta ({\varphi _i},S,t) \le 0) :=-\sqrt[m]{{\prod\limits_{i = 1,...,m} {\left( {1 - \eta ({\varphi _i},S,t)}  \right)}}} + 1\\
\\
\eta(\mathbf{G}_{[a,b]}\varphi, S, t \mid \exists t_k' \in [t+a, t+b] \ .\ \eta ({\varphi},S,t_k')\le 0) :=
\frac{1}{N}\sum\limits_{t_k'\in [t+a,t+b]} [\eta(\varphi, S, t_k')]_-\\
\\
\eta(\mathbf{F}_{[a,b]}\varphi, S, t \mid \forall t_k' \in [t+a, t+b] \ .\ \eta ({\varphi},S,t_k')\le 0 ) := 
-\sqrt[N]{ {\prod\limits_{t_k' \in [t+a,t+b] } {\left( {1 - \eta (\varphi, S, t_k')} \right)}}} + 1
\end{array}
\end{equation}
\end{figure*}
\end{definition}
Same as before, when the time of satisfaction is not mentioned, satisfaction at time $0$ is considered, i.e., $\eta(\varphi,S)=\eta(\varphi,S,0)$. We first find robustness for each individual subformula using \eqref{eq:agm-def-base-cases}. \textbf{Algorithm \ref{alg: and}}, \textbf{Algorithm \ref{alg: or}}, \textbf{Algorithm \ref{alg: G}} and \textbf{Algorithm \ref{alg: F}} then determine satisfaction or violation of the specification with respect to the signal $S$ as well as the normalized signed AGM robustness score.

\begin{remark}
The command $ANY$ used in the AGM robustness algorithms is employed to check satisfaction or violation of the specification and determine the resulting robustness score, for which the worst case complexity is ${\rm O}(n)$.
\end{remark}\vspace{2pt}
\begin{theorem}[Soundness]
\label{th:soundness}
The AGM robustness score is sound, meaning that a satisfying trajectory has a strictly positive robustness:
\begin{equation}
\label{eq:sound}
\begin{aligned}
\eta(\varphi,S,t)>0 \Leftrightarrow \rho(\varphi,S,t)>0
\Rightarrow S\models\varphi,
\\
\eta(\varphi,S,t)<0 \Leftrightarrow \rho(\varphi,S,t)<0 
\Rightarrow S\not\models\varphi.\\
\end{aligned}
\end{equation}
\end{theorem}
\begin{proof}
We prove the property by structural induction over the formula $\varphi$.
The {\em base case} corresponding to $\varphi \in \{\top, \bot, \mu\}$
is trivially true by definition from~\eqref{eq:agm-def-base-cases}.
\\
Let S be a signal. We have the following {\em induction cases}:
\\
{\em Negation:} Let $\phi=\lnot \varphi$ and $\eta(\phi, S, t) > 0$.
We have $\eta(\varphi, S, t) < 0$, and by the induction hypothesis
$S\not\models\varphi$. Thus, $S \models \phi$.
Similarly, for $\eta(\phi, S, t) < 0$ we get $S\nmodels\phi$.
\\
{\em Conjunction:} Let $\phi=\varphi_1 \land \varphi_2$ and $\eta(\phi, S, t) > 0$.
Assume that one or both $\eta(\varphi_i, S, t) < 0$, $i=1, 2$,
then from~\eqref{eq:agm-unsat} we get $\eta(\phi, S, t) = \frac{1}{2}\sum\limits_{i=1,2} [\eta(\varphi_i, S, t)]_- < 0$
which contradicts the assumption. 
It follows that $\eta(\varphi_i, S, t) > 0$, $i=1, 2$.
By the induction hypothesis $S\models \varphi_i$, $i=1, 2$, and thus
$S\models \phi$.
For the case $\eta(\phi, S, t) < 0$, assume $\eta(\varphi_i, S, t) > 0$, $i=1, 2$.
From~\eqref{eq:agm-sat} it follows that $\eta(\phi, S, t) = \sqrt[2]{\prod\limits_{i = 1,2} \left( {1 + \eta ({\varphi _i},S,t)} \right)} - 1 > 0$ which is a contradiction.
Thus, we have either $\eta(\varphi_1, S, t) < 0$ or $\eta(\varphi_2, S, t) < 0$ or both.
Again by the induction hypothesis $S\not\models \varphi_1$ or $S\not\models \varphi_2$, and thus $S\not\models\phi$.\\
{\em Disjunction:} Follows similarly to {\em conjunction} case.
\\
{\em Globally:} Let $\phi = \mathbf{G}_{[a, b]} \varphi$, and $\eta(\phi, S, t) > 0$.
Assume that there is $t_k' \in [t+a, t+b]$ such that $\eta(\varphi, S, t_k') < 0$,
then from~\eqref{eq:agm-unsat} we get $\eta(\phi, S, t) = \frac{1}{N}\sum\limits_{t_k'\in [t+a,t+b]} [\eta(\varphi, S, t_k')]_- < 0$ which contradicts $\eta(\phi, S, t) > 0$.
It follows that $\eta(\varphi, S, t_k') > 0$, $\forall t_k' \in [t+a,t+b]$.
By the induction hypothesis $S[t_k']\models \varphi$, $\forall t_k' \in [t+a,t+b]$,
and thus $S\models \phi$.
For the case $\eta(\phi, S, t) < 0$, assume that for all 
$t_k'\in [t+a,t+b]$, $\eta(\varphi, S, t_k') > 0$.
From~\eqref{eq:agm-sat} we have $\eta(\phi, S, t) = \sqrt[N]{\prod\limits_{t_k' \in [t+a,t+b] } \left( 1 + \eta (\varphi, S, t_k') \right)} - 1 > 0$ which is a contradiction.
Thus, we have $\eta(\varphi, S, t_k') < 0$ for some $t_k'\in [t+a,t+b]$.
Again by the induction hypothesis $S[t_k'] \not\models\varphi$, and thus $S\not\models\phi$.
\\
{\em Eventually:} Follows similarly to the {\em globally} case.
\end{proof}\vspace{2pt}
\begin{proposition}
Let $S$ be a signal and $\phi$ a STL formula.
If $\eta(\phi, S, t) = 1$, then $\eta(\varphi, S, t_k) = 1$ for all
subformulae $\varphi$ of $\phi$ and
appropriate times $t_k$ as given by~\eqref{eq:agm-sat},~\eqref{eq:agm-unsat}.
Similarly, if $\eta(\phi, S, t) = 0$, then $\eta(\varphi, S, t_k) = 0$ and if $\eta(\phi, S, t) = -1$, then $\eta(\varphi, S, t_k) = -1$ for all
subformulae $\varphi$ of $\phi$ and
appropriate times $t_k$ in~\eqref{eq:agm-sat},~\eqref{eq:agm-unsat}.
\end{proposition}
\begin{proof}
The proof is similar to \textit{Theorem~\ref{th:soundness}}.
\end{proof}
\subsection{Logic properties}
Let $\conj: [-1, 1] \times [-1, 1] \to [-1, 1]$ be a conjunction function
defined such that $\conj(\eta(\varphi_1, 
S), \eta(\varphi_2, S)) = \eta(\varphi_1 \land \varphi_2, S)$ for all STL formulae $\varphi_1, \varphi_2$ and signal $S$.
Explicitly,
\begin{equation}
\conj(x, y) = \begin{cases}
\sqrt{(1+x)(1+y)} -1 & x> 0, y> 0\\
\frac{[x]_-+[y]_-}{2} &  \text{else}
\end{cases}
\end{equation}
\begin{proposition}
\label{th:conj-lprop}
The conjunction function satisfies:
\begin{align}
\conj(x, y) = \conj(y, x) & \quad \text{(Commutativity)}\\
\conj(x, y) \leq \conj(u, v), \forall x\leq u, y\leq v & \quad \text{(Monotonicity)}\\
\conj(x, x) = x & \quad \text{(Idempotence)}
\end{align}
\end{proposition}

Similarly, we define the disjunction function $\disj(\cdot, \cdot)$ which also satisfies the same properties in \textit{Proposition \ref{th:conj-lprop}}.
\begin{remark}
A weaker form of absorption with respect to maximum true and minimum false hold
for conjunction $\conj(x, -1) < 0$ and disjunction $\disj(x, 1) > 0$
for all $x \in (-1, 1)$, respectively.
\end{remark}

Let $\nfnc: [-1, 1] \to [-1, 1]$ be the negation function defined such that
$\nfnc(\eta(\varphi, S)) = \eta(\lnot\varphi, S)$ for all STL formula $\varphi$ and signal $S$. Explicitly, $\nfnc(x) = -x$.

Lastly, we define the implication function $\impl:[-1, 1] \times [-1, 1] \to [-1, 1]$
as $\impl(x, y) = \disj(-x, y)$.
\begin{theorem}[Rules of Inference]
The following hold:
\begin{enumerate}
\item {\em Law of non-contradiction:} $\conj(x, \nfnc(x)) < 0$, $\forall x \neq 0$;
\item {\em Law of excluded middle:} $\disj(x, \nfnc(x)) > 0$, $\forall x \neq 0$;
\item {\em DeMorgan's law:} $\disj(x, y) = \nfnc(\conj(\nfnc(x), \nfnc(y)))$, $\forall x, y$;
\item {\em Double negation:} $\nfnc(\nfnc(x)) = x$, $\forall x$;
\item {\em Modus ponens:} if $\impl(x, y) > 0$ and $x > 0$ then $y > 0$.
\end{enumerate}
\end{theorem}
\begin{proof}
All properties follow directly from the definitions.
\end{proof}
\begin{remark}
Although $([-1, 1], \conj, \disj)$ is not a distributive lattice,
i.e., Boolean algebra, it does satisfy the Kleene algebra condition:
$\conj(x, \nfnc(x)) \leq \disj(y, \nfnc(y))$, $\forall x, y \in [-1, 1]$.
\end{remark}
\subsection{Performance Properties}
\begin{properties}[Smoothness and Gradient]
The AGM robustness $\eta(\phi,S,t)$ is smooth in $S \in [-1, 1]^n$ almost everywhere
except on the satisfaction boundaries $\rho(\varphi,S,t_k)=0$,
where $\varphi$ is a subformula of $\phi$, and appropriate times $t_k$
as given in~\eqref{eq:agm-sat} and~\eqref{eq:agm-unsat}.
Moreover, the gradient of $\eta$ with respect to the elements of $S$ that are
part of $\phi$'s predicates is non-zero wherever it is smooth.
\end{properties}
\begin{proof}[Sketch]
The property follows by structural induction over the formula $\phi$,
and the smoothness and non-zero gradient of the conjunction $\conj$ and
disjunction $\disj$ functions on $\big((-1, 1)\setminus \{0\}\big)^2$,
and negation $\nfnc$ on $(-1, 1)$.
The cases for the {\em globally} and {\em eventually} operators follow similarly.
\end{proof}
\begin{algorithm}[!t]
\caption{\textsc{AGM Robustness for And}}
\label{alg: and}
\KwIn{STL Formula $\phi= \varphi_1 \wedge \varphi_2 \wedge ... \varphi_m $; Signal $S$}
\KwOut{AGM Robustness $\eta(\phi,S)$}
Find $\eta(\varphi_i,S)$ for $i=\{1,2,...,m\}$ using \eqref{eq:agm-def-base-cases}\;
If ${ANY}(\eta(\varphi_i,S)\le 0 )$, then $S  \nmodels \phi$, 
$\;\;\;\;\;\;\;\;\;\;\;\;\;\;\;\;\;\eta (\phi ,S|S \nmodels \phi) := \frac{1}{m} \sum\limits_{i=1,...,m} {[\eta(\varphi_i,S)]_-}$\;
Else: $S  \models \phi$, 
$\eta(\phi,S| S \models \phi):=\sqrt[m]{ {\prod\limits_{i = 1,...,m} {\left( {1 + \eta ({\varphi _i},S)} \right)} }} - 1$. 
\end{algorithm} 
\begin{algorithm}[!t]
\caption{\textsc{AGM Robustness for Or}}
\label{alg: or}
\KwIn{STL Formula $\phi= \varphi_1 \vee \varphi_2 \vee ... \varphi_m $; Signal $S$}
\KwOut{AGM Robustness $\eta(\phi,S)$}
Find $\eta(\varphi_i,S)$ for $i=\{1,2,...,m\}$ using \eqref{eq:agm-def-base-cases}\;
If ${ANY}(\eta(\varphi_i,S) > 0 )$, then $S  \models \phi$, 
$\;\;\;\;\;\;\;\;\;\;\;\;\;\;\;\;\;\eta (\phi ,S|S \models \phi) := \frac{1}{m} \sum\limits_{i=1,...,m} {[\eta(\varphi_i,S)]_+}$\;
Else: $S  \nmodels \phi$, 
$\eta(\phi,S|S \nmodels \phi)=-\sqrt[m]{{ {\prod\limits_{i = 1,...,m} {\left( {1 - \eta ({\varphi _i},S)} \right)}}}} + 1$.
\end{algorithm}
\begin{algorithm}[!t]
\caption{\textsc{AGM Robustness for Globally}}
\label{alg: G}
\KwIn{STL Formula $\phi=\mathbf{G}_{[a,b]} \varphi$; Signal $S$}
\KwOut{AGM Robustness $\eta(\phi,S)$}
Find $\eta(\varphi,S[t_k'])$ for time points $t_k' \in [a,b]$ using \eqref{eq:agm-def-base-cases}\;
If 
${ANY}(\eta(\varphi, S[t_k']) \le 0 $), then $S  \nmodels \phi$, $\;\;\;\;\;\;\;\;\;\;\;\;\;\;\;\;\;\eta (\phi ,S|S \nmodels \phi) := \frac{1}{N}\sum\limits_{t_k'\in [a,b]} [\eta(\varphi, S[ t_k'])]_-$\;
Else: $S  \models \phi$,
$\eta(\phi,S| S \models \phi):= \sqrt[N]{ {\prod\limits_{t_k' \in [a,b] } {\left( {1 + \eta (\varphi, S[ t_k'])} \right)} } } - 1.$
\end{algorithm}
\begin{algorithm}[!t]
\caption{\textsc{AGM Robustness for Eventually}}
\label{alg: F}
\KwIn{STL Formula $\phi=\mathbf{F}_{[a,b]}\varphi$; Signal $S$}
\KwOut{AGM Robustness $\eta(\phi,S)$}
Find $\eta(\varphi,S[t_k']$ for time points $t_k' \in [a,b]$ using \eqref{eq:agm-def-base-cases}\;
If ${ANY}(\eta(\varphi,S[t_k'])> 0 $), then $S \models \phi$, $\;\;\;\;\;\;\;\;\;\;\;\;\;\;\;\;\;\;\;\;\;\eta (\phi ,S|S \models \phi) := \frac{1}{N}\sum\limits_{t_k' \in[a,b] } [\eta(\varphi, S[t_k'])]_+ $ \;
Else: $S  \nmodels \phi$,
$\eta (\phi ,S|S \nmodels \phi) :=-\sqrt[N]{ {\prod\limits_{t_k' \in [a,b] } {\left( {1 - \eta (\varphi, S[t_k'])} \right)}}} + 1.$
\end{algorithm} \vspace{3pt}
\begin{properties}[Arithmetic and Geometric Means]
By employing arithmetic and geometric means for defining the AGM robustness, we can measure how well a specification $\phi$ is satisfied, taking in to account robustness for all the subformulae $\varphi$ of $\phi$ or at all appropriate times $t_k$ and not just the most critical one with maximum/minimum satisfaction. Comparison between traditional and AGM robustness scores demonstrates the advantage of our average-based definition. For instance, consider a signal $S \in [0,1]$ and three subformulae $\varphi_1,\varphi_2,\varphi_3$ with $\rho(\varphi_1,S)=\rho(\varphi_2,S)=\eta(\varphi_1,S)=\eta(\varphi_2,S)=1$ and $\rho(\varphi_3,S)=\eta(\varphi_3,S)=0.2$. While traditional robustness uses $\max$ function and returns $\rho(\varphi_1 \vee \varphi_2,S)=\rho(\varphi_1 \vee \varphi_3,S)=1$; AGM definition returns $\eta(\varphi_1 \vee \varphi_3,S)=0.6$, which is positive showing that the specification is satisfied, but the robustness is less than $1$ (highest satisfaction), which is attainable only when both subformulae are maximally satisfied, i.e., $\eta(\varphi_1 \vee \varphi_2,S)=1$.
We now assume $\rho(\varphi_1,S)=\rho(\varphi_2,S)=\eta(\varphi_1,S)=\eta(\varphi_2,S)=0.2$ and $\rho(\varphi_3,S)=\eta(\varphi_3,S)=1$. While traditional robustness uses $\min$ function and returns $\rho(\varphi_1 \wedge \varphi_2,S)=\rho(\varphi_1 \wedge \varphi_3,S)=0.2$; AGM definition returns $\eta(\varphi_1 \wedge \varphi_2,S)=0.2$, which is positive showing the specification is satisfied, but the robustness is less than $\eta(\varphi_1 \wedge\varphi_3,S)=0.55$, which shows a stronger satisfaction. Now consider three signals $S_1,S_2,S_3$ illustrated in Fig. 1. We first examine traditional and AGM robustness score for $\phi_1=\mathbf{F}_{[1,4]}(S > 0.5)$. Using $\max$ function in the traditional definition, $\rho (\phi_1,S_i)=0.5$ for $i=1,2,3$ in Fig. 1 (Left). However, AGM robustness takes a time average over the formula horizon considering all the times the predicate is satisfied; therefore, it returns higher robustness $\eta(\phi_1,S_1)=0.5$ for $S_1$ and lower robustness $\eta(\phi_1,S_2)=0.25$ and $\eta(\phi_1,S_3)=0.125$ for $S_2$,$S_3$, respectively. Basically, AGM robustness for $F_{[a,b]}\varphi$ can be interpreted as \enquote{\textit{eventually satisfy $\varphi$ with the maximum possible satisfaction as early as possible and for as long as possible}}. For signals in Fig. 1 (Right) and $\phi_2=\mathbf{G}_{[0,4]} (S > 0.5)$, traditional robustness with $\min$ function returns $\rho(\phi_2,S_1)=0.5$ for $S_1$, while giving same robustness $\rho (\phi_2,S_i)=0.1$ for $S_2,S_3$. On the other hand, AGM definition calculates $\eta (\phi_2,S_1)=0.5$ for $S_1$, and lower robustness scores $\eta(\phi_2,S_2)=0.41$ for $S_2$ and $\eta (\phi_2,S_3)=0.1$ for $S_3$. Thus, the AGM definition for $G_{[a,b]}\varphi$ can be interpreted as \enquote{\textit{always satisfy $\varphi$ with the maximum possible satisfaction for all the time points in $[a,b]$}}.
\begin{figure}[htb]
\centering
\begin{tabular}{cc}
%
%
\definecolor{mycolor1}{rgb}{0.00000,0.44700,0.74100}%

\definecolor{mycolor3}{rgb}{0.85000,0.32500,0.09800}%

\definecolor{mycolor5}{rgb}{0.00000,0.74902,0.74902}%
\pgfplotsset{tick label style={font=\tiny}}
\begin{tikzpicture}

\begin{axis}[%
width=2.9cm,
height=2.3cm,
at={(0cm,0cm)},
scale only axis,
xmin=0,
xmax=4,
xtick={0, 1, 2, 3, 4},
xlabel={\tiny time},
ymin=0,
ymax=1.2,
ytick={  0, 0.5,0.9, 1, 1.4},
axis background/.style={fill=white},
xmajorgrids,
ymajorgrids,
legend style={fill=none,draw=none, font=\tiny  ,at={(1.07,0.4570806)},  align=right}
]
\addplot [color=mycolor1, dotted, line width=1.3pt, mark=*, mark options={solid, fill=mycolor1,mark size=1.5pt}]
  table[row sep=crcr]{%
0	0\\
1	1\\
2	1\\
3	1\\
4	1\\
};

\addplot [color=mycolor3, dashdotted, line width=1.0pt, mark=pentagon*, mark options={solid, fill=mycolor3,mark size=1.1pt}]
  table[row sep=crcr]{%
0	0\\
1	0.3\\
2	0.3\\
3	1\\
4	1\\
};

\addplot [color=mycolor5, dashed, line width=1.0pt, mark=square*, mark options={solid, fill= mycolor5,mark size=.9pt}]
  table[row sep=crcr]{%
0	0\\
1	1\\
2	0.3\\
3	0.3\\
4	0.3\\
};

\end{axis}
\end{tikzpicture}%
&
%
%
\definecolor{mycolor1}{rgb}{0.00000,0.44700,0.74100}%

\definecolor{mycolor3}{rgb}{0.85000,0.32500,0.09800}%

\definecolor{mycolor5}{rgb}{0.00000,0.74902,0.74902}%
\pgfplotsset{tick label style={font=\tiny}}
\begin{tikzpicture}

\begin{axis}[%
width=2.9cm,
height=2.3cm,
at={(0cm,0cm)},
scale only axis,
xmin=0,
xmax=4,
xtick={0, 1, 2, 3, 4},
xlabel={\tiny time},
ymin=0,
ymax=1.2,
ytick={  0, 0.5, 0.6, 1},
axis background/.style={fill=white},
xmajorgrids,
ymajorgrids,
legend style={fill=none,draw=none,font=\tiny, at={(0.4950811,0.46806)}, align=left}
]
\addplot [color=mycolor1, dotted, line width=1.3pt, mark=*, mark options={solid, fill=mycolor1,mark size=1.5pt}]
  table[row sep=crcr]{%
0	1\\
1	1\\
2	1\\
3	1\\
4	1\\
};
\addlegendentry{$S_1$}

\addplot [color=mycolor3, dashdotted, line width=1.0pt, mark=pentagon*, mark options={solid, fill=mycolor3,mark size=1.1pt}]
  table[row sep=crcr]{%
0	0.6\\
1	1\\
2	1\\
3	1\\
4	1\\
};
\addlegendentry{$S_2$}

\addplot [color=mycolor5, dashed, line width=1.0pt, mark=square*, mark options={solid, fill= mycolor5,mark size=0.9pt}]
  table[row sep=crcr]{%
0	0.6\\
1	0.6\\
2	0.6\\
3	0.6\\
4	0.6\\
};
\addlegendentry{$S_3$}

\end{axis}
\end{tikzpicture}%
\end{tabular}
\caption{Signals for traditional and AGM robustness comparison.}
\end{figure}
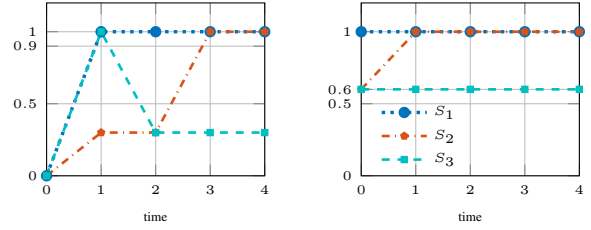 
\end{properties}
\begin{properties}[Performance Under Disturbance]
The AGM robustness score provides a better satisfaction margin in the presence of disturbance. Consider the specification $\phi_3= F_{[1,4]}(S>0.9)$ and signals $S_1,S_3$ in Fig. 1 (Left), satisfying $\phi_3$ with the same traditional and AGM robustness score $\rho(\phi_3,S_3)=\eta(\phi_3,S_1)=0.1$. In the traditional robustness definition, $S_3$ only satisfies $\phi_3$ at a single time point, i.e., at $t=1$ with $\rho(\phi_3,S_3[1])=0.1$; while for $S_1$ to have the same score using AGM robustness, $\eta(\phi_3,S_1[t_k])=0.1$ for all $t_k \in [1,4]$. It can be easily shown that applying any disturbance $d > 0.1$ at $t=1$ to $S_3$ results in violation of $\phi_3$. However, $\phi_3$ is still satisfied in $S_1$ under the same disturbance $d$, although the satisfaction would become weaker. Therefore, at a same score for the traditional and AGM robustness, satisfaction would hold for larger disturbance using the AGM definition. 
\end{properties}
\subsection{Normalization}
The normalization with respect to the range of the elements of $S$ is not restrictive but is desired to provide a meaningful understanding about satisfaction or violation of a specification, especially when comparing robustness in a formula with predicates defined over different properties or scales. For instance, consider the following specification:\vspace{-1pt}  
\[
  \label{spec}
 \begin{array} {l}

\varphi_1= x_{robot} > 5,\\
\varphi_2=\textit{Battery}> 30,\\
\phi= \varphi_1 \wedge \varphi_2,
 \end{array}
\]
where $x_{robot} \in [0,10]$ is the position of the robot and $\textit{Battery} \in [0,100]$ shows its battery level. Without normalization, at $x_{robot}=6,\; \textit{Battery} = 80$, robustness $\rho(\varphi_1,x_{robot})=1$ and $\rho(\varphi_2,\textit{Battery})=50$. Since the variables are in different scales, unnormalized robustness score is not a meaningful measure of how well the specification $\phi$ is satisfied, i.e., we have $\rho(\phi,(x_{robot},\textit{Battery}))= 1$ for $x_{robot}=6$, and any $\textit{Battery} = \{31,32,...,100\}$. Therefore, not only normalization is not limiting, but is actually essential in practice.

\section{CONTROL USING THE AGM ROBUSTNESS}
To solve the control synthesis problem \eqref{eq: cost}, we need to find optimal trajectories which satisfy the specification $\phi$. A positive robustness score provides a margin in which any perturbation up to $\eta$ does not change satisfaction of the specification. Therefore, we can maximize robustness over all possible control inputs to find not just a satisfying trajectory, but one that has the strongest satisfaction of the specification:
\begin{equation}
\label{eq: rob}
\begin{array}{c}
u^*=\argmax_u \eta(\phi,\langle q,u\rangle)\\
\; \text{s.t.}\;\;\;\;\; \eta(\phi,\langle q,u\rangle)>0.
\end{array}
\end{equation}
Assume the system dynamics $f$ in \eqref{eq:dynamics} is smooth. Based on \textit{Property 1}, we can use advanced optimization methods such as gradient ascent to maximize the AGM robustness $\eta$, rather than using heuristic methods or MILP encoding. Gradient ascent is an iterative optimization algorithm for finding maximum of a function $F(x)$ by taking steps proportional to gradient of the function at each iteration $i$: 
\begin{equation}
\label{eq: GA}
x^{i+1}\gets x^{i}+\alpha^i\; \nabla F,
\end{equation}
where $\nabla F = \frac{\partial F}{\partial x}$ and $\alpha^i$ is step size at iteration $i$. Despite heuristic optimization algorithms which have so many parameters to be set, gradient methods only need to tune the step size $\alpha$. Due to non-smoothness in $\eta$ at the satisfaction boundaries, we use proximal stochastic gradient ascent or sub-gradient ascent method with diminishing step size \cite{GAbook}. To initialize gradient ascent, a random control input sequence $u^0 \in \mathbf{U}$ is generated, and the resulting trajectory starting from initial state $q_0$ is found using system dynamics, which may violate the state constraints or STL specification. The gradient ascent optimization then finds optimal control policy $u^*$ which maximizes AGM robustness function $\eta$ for given STL constraints $\phi$ with respect to the system execution $\langle q,u\rangle$. \\
Combining \eqref{eq: rob} and \eqref{eq: cost}, we can solve a relaxed problem in which we maximize the robustness as much as possible as well as minimizing the penalized cost. The combined fitness function is defined as:
\begin{equation}
\label{eq: opt}
\begin{array}{c}
u^*={\argmax}_u (\eta(\phi,\langle q,u\rangle)-\lambda \sum\limits_{k=0}^{T-1} J(u[k],q[k+1])),\\
\text{s.t.}\;\;\;\; \eta(\phi,\langle q,u\rangle)>0,\\
\;\;\;\;\;\;\;\;\;\;\;\;\;\;q[k+1]=f(q[k],u[k]),\\
q[0]=q_0,\\
q[k] \in \mathbf{Q} \subseteq\mathbb{R}^n,\\
u[k] \in \mathbf{U} \subseteq\mathbb{R}^m,
\end{array}
\end{equation}
where $\lambda$ penalizes the trade-off between maximizing robustness to get the highest STL satisfaction and minimizing the associated cost. Assuming the cost function $J$ is also smooth, a similar gradient ascent optimization can be used to solve the constrained nonlinear optimization problem \eqref{eq: opt}.
\section{CASE STUDIES}
\label{case}
In this section, we show the applicability and efficacy of our framework for control synthesis problems in both linear and nonlinear systems with and without external disturbance, and compare our results with the MILP approach for traditional robustness and SQP approach for the approximation robustness. To emphasize the differences between the proposed robustness and the traditional and approximation ones, we set $\lambda=0$ in \eqref{eq: opt}. Gradient ascent simulations are coded in MATLAB and MILP is implemented in the Gurobi package in Python. 
The maximum number of iterations for gradient ascent is set to $300$. 

\subsection{AGM Robustness Versus Traditional Robustness}
\begin{problem}
Consider a nonholonomic dynamical system:
\begin{equation}
\label{eq:w}
\begin{array}{l}
x[k+1]=x[k]+\cos\theta[k]v[k],\\
y[k+1]=y[k]+\sin\theta[k]v[k],\\
\theta[k+1]=\theta[k]+w[k],\\
\end{array}
\end{equation}
and the desired task \enquote{\textit{Always} stay in the \textit{Init} for $5$ steps $and$ \textit{eventually} visit \textit{Reg1} between $[6,10]$ steps $and$ \textit{eventually} visit \textit{Reg2} between $[11,15]$ steps $and$ \textit{Always} avoid \textit{Obs}}, formally specified as STL formula:
\begin{equation}
\begin{array}{l}
\label{eq:ex1}
\phi_1=(\mathbf{G}_{[1,5]}
\;\textit{Init}) \wedge \;(\mathbf{F}_{[6,10]}\;\textit{Reg1})\\\;\;\;\;\;\;\;\; \wedge\;(\mathbf{F}_{[11,15]} \;\textit{Reg2}) \;\wedge (\mathbf{G}_{[0,15]}\;\neg\textit{Obs}),
\end{array}
\end{equation}
where $\textit{Obs}=[4,7] \times [4,8]$ is the obstacle to avoid, $\textit{Reg1}=[5,7] \times [0,3]$ and $\textit{Reg2}=[8,10] \times [4,6]$ are regions to be sequentially visited, and $\textit{Init} =[0,3] \times [3,7]$ is the region containing the initial position. The state vector $q=[x,y,\theta]$ indicates the robot position and orientation with $\mathbf{Q}=[0,10]^2 \times[-2\pi,2\pi]$, initial state is $q_0=[0.5,5,0]$ and $u=[v,w]$ is the input vector with $\mathbf{U}=[-1.3,1.3]^2$.
\end{problem}

To maximize the traditional robustness $\rho$ using the MILP implementation, we need to linearize the dynamics. We use feedback linearization to convert the nonlinear dynamics \eqref{eq:w} in to a discrete double integrator dynamics \cite{milp}:
\begin{equation}
\label{eq:feedLin}
\begin{array}{l}
q'[k+1]=q'[k]+q'_{d}[k],\\
q'_{d}[k+1]=q'_{d}[k]+u_{q'}[k],
\end{array}
\end{equation}
with $q'=[x,y]$ being the new state vector, $q'_d$ the first order discrete derivative, and $u_{q'}=[u_x,u_y]$ the new control inputs for the linearized system that we synthesize. Therefore, by linearizing dynamics, we can only control $x$ and $y$ directly and robot orientation $\theta$ is controlled indirectly. Two optimal trajectories maximizing traditional robustness for the linearized system \eqref{eq:feedLin} 
found by Gurobi with same maximum traditional robustness $\rho=1$ are shown in Fig. \ref{fig: milp}. The MILP implementation for STL constraints in $\phi_1$ with time horizon $T=15$ has $95$ continuous and $70$ integer (binary) variables. It is shown in \cite{husam} that MILP does not scale well with the number of integer variables. Therefore, MILP is not applicable for complex specifications with many $\lor$ and $\mathbf{F}$ operators (that must be encoded as binary variables) and long time horizons. 

We next maximize AGM robustness $\eta$ for the nonlinear dynamics \eqref{eq:w} using gradient ascent. Fig. \ref{fig:nw} shows two trajectories satisfying STL constraints in $\phi_1$ obtained in different iterations of gradient ascent. Although both methods generate satisfying trajectories, our proposed approach generates a more smooth trajectory by controlling both robot position and orientation. Moreover, maximum traditional robustness using MILP is obtained when trajectory visits each region with maximum robustness at a single time point ($Reg1$ at $t=10$, $Reg2$ at $t=15$) without rewarding the frequency of satisfaction while using the AGM robustness, trajectory with higher robustness visits $Reg1$ \textit{as early as possible and for as long as possible} ($t=9,10$); with all subformuale having \textit{maximum possible robustness} (trajectory is toward the center of regions) while always avoiding obstacle. 

\subsection{AGM Robustness Versus Approximation Robustness}
In \cite{husam} authors used Sequential Quadratic Programming (SQP) on the smooth approximation robustness $\tilde \rho$ of MTL specifications and showed it was more time efficient than MILP approach. However, smooth approximation was within a pre-defined error $\delta$ of the traditional robustness, i.e., $|\rho -\tilde \rho| \le \delta$. As a result, a positive approximation robustness $\tilde \rho$ did not necessarily correspond to a trajectory satisfying the specification and it was required to add $\tilde \rho \ge \delta$ as a constraint in the optimization problem. We compare the results for maximizing approximation robustness $\tilde {\rho}$ and AGM robustness $\eta$ and show the advantage of our approach, both in accuracy (removing errors due to soft minimum/maximum approximations) and satisfaction performance.
\begin{problem} Consider the nonlinear dynamical system:
\begin{equation}
\label{eq:nonlinear}
\begin{array}{l}
x[k+1]=x[k]+\cos\theta[k]v[k],\\
y[k+1]=y[k]+\sin\theta[k]v[k],\\
\theta[k+1]=\theta[k]+v[k]w[k],\\
\end{array}
\end{equation}
and the desired task \enquote{\textit{Eventually} visit \textit{Reg1} $or$ \textit{Reg2} between $[1,5]$ steps $and$ \textit{eventually} visit $\textit{Reg3}$ between $[6,10]$ steps $and$ \textit{Always} avoid $Obs$}, formally specified as STL formula:
\begin{equation}
\label{eq:ex2}
\begin{array}{l}
\phi_2=(\mathbf{F}_{[1,5]}\;(\textit{Reg1}\vee \textit{Reg2})) \;\wedge(\mathbf{F}_{[6,10]} \textit{ Reg3}) \\ \;\;\;\;\;\;\;\;\;\;\;
\wedge\;(\mathbf{G}_{[0,10]}\;\neg\textit{Obs}),
\end{array}
\end{equation}
where $\textit{Obs}=[3,6]^2$ is the obstacle to avoid, $\textit{Reg1}=[5,7] \times [1,3]$ or $\textit{Reg2}=[1,3] \times [6,8]$ and $\textit{Reg3}=[7,10] \times [7,9]$ are regions to be sequentially visited. State vector $q=[x,y,\theta]$ indicates robot position and orientation with $\mathbf{Q}=[0,10]^2 \times[-2\pi,2\pi]$ and initial state $q_0=[1,1,0]$, and $u=[v,w]$ is the input vector with $U=[-2,2]^2$.
\end{problem}
\begin{figure}[t]
\centering
\begin{tabular}{cc}
\includegraphics{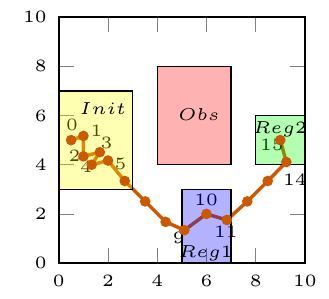}&
\includegraphics{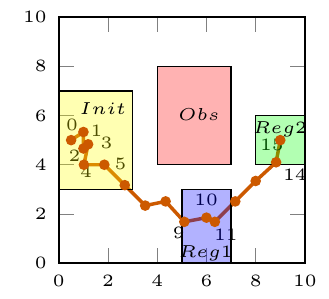}
 \end{tabular}
 \caption{Trajectories with same maximum traditional robustness $\rho=1$ found by Gurobi.}
\label{fig: milp}\vspace{-5pt} 
\end{figure} 
\begin{figure}[t]
\centering
\begin{tabular}{cc}
\includegraphics{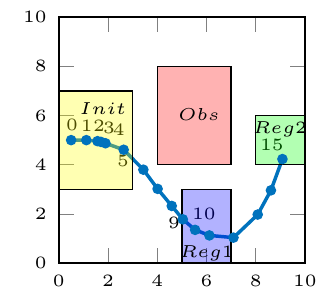}&
\includegraphics{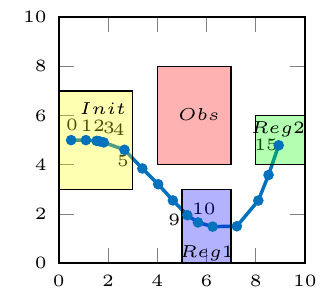}
 \end{tabular}
 \caption{Trajectory with positive AGM robustness $\eta=0.138$ 
 (Left) and after more gradient ascent iterations with $\eta=0.144$ (Right).}
\label{fig:nw}\vspace{-5pt} 
\end{figure} 
\begin{figure}[t]
\centering
\begin{tabular}{cc}
\includegraphics{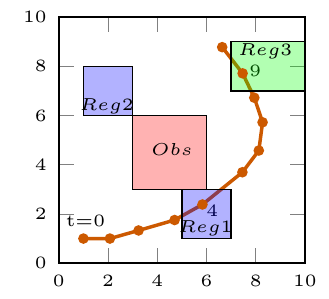}&
\includegraphics{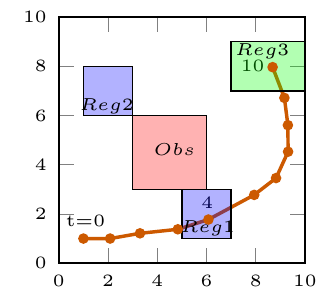}
 \end{tabular}
 \caption{Trajectory with positive approximation robustness $\tilde{\rho}=0.424 \; (\rho=0.461)$ (Left) and after more gradient ascent iterations $\tilde{\rho}=0.731 \; (\rho=0.778)$ (Right)}
\label{fig:awv}\vspace{-5pt} 
\end{figure}  
\begin{figure}[t]
\centering
\begin{tabular}{cc}
\includegraphics{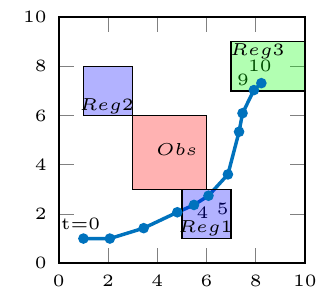}&
\includegraphics{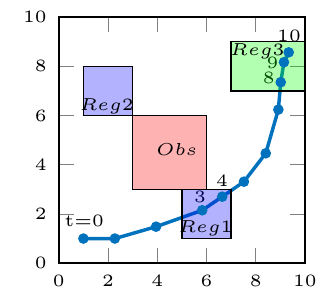}
 \end{tabular}
 \caption{Trajectory with positive AGM robustness $\eta=0.130$ (Left) and after more gradient ascent iterations with $\eta=0.171$ (Right).}
\label{fig:nwv}\vspace{-5pt} 
\end{figure}  
Fig. \ref{fig:awv} and Fig. \ref{fig:nwv} show trajectories satisfying STL constraints in $\phi_2$ obtained using gradient ascent maximizing the approximation robustness $\tilde{\rho}$ with $\beta=10$ in \eqref{eq:soft} and the AGM robustness $\eta$, achieved up to the termination criteria. Although both methods generate trajectories satisfying the specification $\phi_2$, the trajectory with higher approximation robustness visits $Reg1$ ($t=4$) and $Reg3$ ($t=10$) at a single time point while using the AGM robustness, trajectory with higher robustness visits $Reg1$ ($t=3,4$) and $Reg3$ ($t=8,9,10$) \textit{as early as possible and for as long as possible}; forcing trajectory to move toward the center of each region while always keeping distance $|\eta|$ with the obstacle. As discussed earlier, due to the approximation errors resulted from approximating $\max$ and $\min$ functions, traditional robustness $\rho$ and approximation robustness $\tilde{\rho}$ have different values for the same trajectory. 
\subsection{Performance Under Disturbance}
We demonstrate the advantage of maximizing the AGM robustness, rather than the traditional robustness, in a control synthesis problem under external disturbance. 
\begin{problem} Consider a linear dynamical system: 
\begin{equation}
\label{linear}
\begin{array}{l}
x[k+1]=x[k]+u_x[k],\\
y[k+1]=y[k]+u_y[k],
\end{array}
\end{equation}
where $q=[x,y]$ is the state vector indicating robot position with $\mathbf{Q}=[0,6]^2$ and initial state $q_0=[0,1]$, and $u=[u_x,u_y]$ with $U=[-1.5,1.5]^2$. The desired task is \enquote{\textit{Eventually} visit \textit{Reg1} between $[1,5]$ steps $and$ \textit{eventually} visit \textit{Reg2} between $[6,10]$ steps}, formally specified as STL formula:
\begin{equation}
\label{eq:ex3}
\phi_3=(\mathbf{F}_{[1,5]}\;\textit{Reg1})\; \wedge\;(\mathbf{F}_{[6,10]} \;\textit{Reg2}),
\end{equation}
where $\textit{Reg1}=[1,2] \times [3,4]$ and $\textit{Reg2}=[3,4] \times [1,2]$ are regions to be sequentially visited.
\end{problem}
\begin{figure}[t]
\centering
\begin{tabular}{cc}
\includegraphics{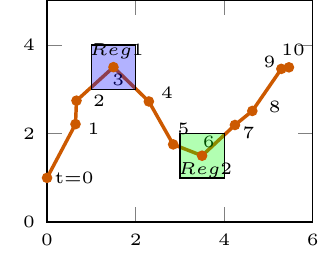}
\includegraphics{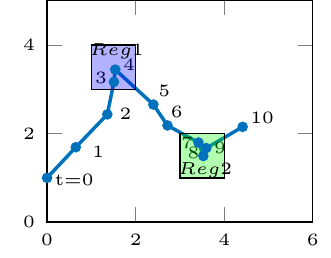}
 \end{tabular}
 \caption{Trajectories with maximum approximation robustness $\tilde{\rho}=0.292$ ($\rho$=0.5) (Left) and positive AGM robustness $\eta=0.173$ (Right).}
\label{fig: lin}
\end{figure} 
\begin{figure}[t]
\centering
\begin{tabular}{cc}
\includegraphics{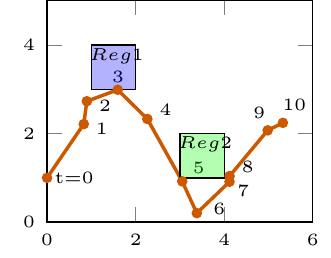}
\includegraphics{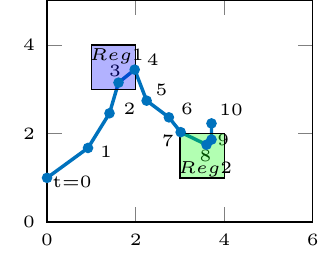}
 \end{tabular}
 \caption{A trajectory generated from disturbed $u_{\tilde{\rho}}^*$ violating $\phi_3$ with negative approximation robustness $\tilde{\rho}=-0.123$ ($\rho=-0.083$) (Left) and a trajectory generated from disturbed $u_{\eta}^*$ satisfying $\phi_3$ with positive AGM robustness $\eta=0.166$ (Right).}
\label{fig: lindis}\vspace{-5pt} 
\end{figure} 
We first find control policies $u^*$ maximizing 
$\tilde {\rho}$ and $\eta$ using gradient ascent. Optimal trajectories satisfying the specification $\phi_3$ are shown in Fig. \ref{fig: lin}. It is clear that maximum approximation robustness 
$\tilde{\rho}$ is achieved when the center of each region is visited for a single time point ($Reg1$ at $t=3$, $Reg2$ at $t=6$). However, by maximizing the AGM robustness $\eta$, \textit{not only the center of each region (maximum satisfaction) is visited at least once, but also each region is visited for more time points} ($Reg1$ at $t=3,4$, $Reg2$ at $t=7,8,9$).

Next, we perturb system by adding a gaussian noise $\mathcal{{N}}(0,\sigma^2)$ to the previously found optimal control policies: 
\begin{equation}
\label{eq: dis}
\begin{array} {l}
u_{\tilde{\rho}}^* \leftarrow  u_{\tilde{\rho}}^* + \mathcal{{N}}(0,\sigma^2),\\
u_{\eta}^* \leftarrow  u_{\eta}^* + \mathcal{{N}}(0,\sigma^2)
\end{array}
\end{equation}
We apply the disturbed control policies \eqref{eq: dis} to the system 
and find the resulting trajectories for different values of $\sigma$ over $100$ simulations. The results show that the disturbed control policy maximizing approximation robustness fails to satisfy the specification in $58\%$ of the times, while by maximizing the AGM robustness, specification fails for an average of $41\%$. Fig. \ref{fig: lindis} (Left) illustrates a resulting trajectory by applying disturbed optimal policy $u_{\tilde{\rho}}^*$ violating $\phi_3$; and (Right) a resulting trajectory by applying disturbed optimal policy $u_{\eta}^*$, still satisfying $\phi_3$ but at different time points and with a smaller robustness score $\eta$. Therefore, the control policy found by maximizing the AGM robustness score performs better when disturbance is added to the system after designing the control input.  
 \section{CONCLUSION AND FUTURE WORK}
\label{sec:conclusion}
We presented a novel average-based robustness score for STL by considering not just the critical subformula or the critical time points but all subformulae at all appropriate time points. We demonstrated that the proposed AGM robustness provides a better satisfaction or violation score both in monitoring and control problems 
compared to the traditional robustness. 
We also showed that the system under external disturbance has, on average, better performance when maximizing the AGM robustness rather than the traditional one.


\bibliographystyle{IEEEtran}
\bibliography{thebibliography}

\begin{thebibliography}{10}
\providecommand{\url}[1]{#1}
\csname url@rmstyle\endcsname
\providecommand{\newblock}{\relax}
\providecommand{\bibinfo}[2]{#2}
\providecommand\BIBentrySTDinterwordspacing{\spaceskip=0pt\relax}
\providecommand\BIBentryALTinterwordstretchfactor{4}
\providecommand\BIBentryALTinterwordspacing{\spaceskip=\fontdimen2\font plus
\BIBentryALTinterwordstretchfactor\fontdimen3\font minus
  \fontdimen4\font\relax}
\providecommand\BIBforeignlanguage[2]{{%
\expandafter\ifx\csname l@#1\endcsname\relax
\typeout{** WARNING: IEEEtran.bst: No hyphenation pattern has been}%
\typeout{** loaded for the language `#1'. Using the pattern for}%
\typeout{** the default language instead.}%
\else
\language=\csname l@#1\endcsname
\fi
#2}}

\bibitem{belta}
C.~Belta, B.~Yordanov, and E.~A. Gol, \emph{Formal methods for discrete-time
  dynamical systems}.\hskip 1em plus 0.5em minus 0.4em\relax Springer, 2017,
  vol.~89.

\bibitem{ltl}
A.~Pnueli, ``The temporal logic of programs,'' in \emph{18th Annual Symposium
  on Foundations of Computer Science}, Oct 1977, pp. 46--57.

\bibitem{mtl}
R.~Koymans, ``Specifying real-time properties with metric temporal logic,''
  \emph{Real-time systems}, vol.~2, no.~4, pp. 255--299, 1990.

\bibitem{stl}
O.~Maler and D.~Nickovic, ``Monitoring temporal properties of continuous
  signals,'' in \emph{Formal Techniques, Modelling and Analysis of Timed and
  Fault-Tolerant Systems}.\hskip 1em plus 0.5em minus 0.4em\relax Springer,
  2004, pp. 152--166.

\bibitem{twtl}
C.-I. Vasile, D.~Aksaray, and C.~Belta, ``Time window temporal logic,''
  \emph{Theoretical Computer Science}, vol. 691, pp. 27--54, 2017.

\bibitem{IROS}
C.~I. Vasile, V.~Raman, and S.~Karaman, ``{Sampling-based Synthesis of
  Maximally-Satisfying Controllers for Temporal Logic Specifications},'' in
  \emph{IEEE/RSJ International Conference on Intelligent Robots and Systems
  (IROS)}, Vancouver, BC, Canada, 2017, pp. 3840--3847.

\bibitem{shoukry}
Y.~Shoukry, P.~Nuzzo, A.~Balkan, I.~Saha, A.~L. Sangiovanni-Vincentelli, S.~A.
  Seshia, G.~J. Pappas, and P.~Tabuada, ``Linear temporal logic motion planning
  for teams of underactuated robots using satisfiability modulo convex
  programming,'' in \emph{Annual Conference on Decision and Control
  (CDC)}.\hskip 1em plus 0.5em minus 0.4em\relax IEEE, 2017, pp. 1132--1137.

\bibitem{dorsa}
V.~Raman, A.~Donz{\'e}, D.~Sadigh, R.~M. Murray, and S.~A. Seshia, ``Reactive
  synthesis from signal temporal logic specifications,'' in \emph{Proceedings
  of the 18th International Conference on Hybrid Systems: Computation and
  Control}.\hskip 1em plus 0.5em minus 0.4em\relax ACM, 2015, pp. 239--248.

\bibitem{donze}
A.~Donz{\'e} and O.~Maler, ``Robust satisfaction of temporal logic over
  real-valued signals,'' in \emph{International Conference on Formal Modeling
  and Analysis of Timed Systems}.\hskip 1em plus 0.5em minus 0.4em\relax
  Springer, 2010, pp. 92--106.

\bibitem{cdc}
N.~Mehdipour, D.~Briers, I.~Haghighi, C.~M. Glen, M.~L. Kemp, and C.~Belta,
  ``Spatial-temporal pattern synthesis in a network of locally interacting
  cells,'' in \emph{2018 IEEE Conference on Decision and Control (CDC)}.\hskip
  1em plus 0.5em minus 0.4em\relax IEEE, 2018, pp. 3516--3521.

\bibitem{SA}
H.~Abbas and G.~Fainekos, ``{Convergence proofs for Simulated Annealing
  falsification of safety properties},'' in \emph{Allerton Conference on
  Communication, Control, and Computing}, 2012, pp. 1594--1601.

\bibitem{rrt}
S.~M. LaValle and J.~J. Kuffner~Jr, ``Randomized kinodynamic planning,''
  \emph{The International Journal of Robotics Research}, vol.~20, no.~5, pp.
  378--400, 2001.

\bibitem{raman}
V.~Raman, A.~Donzé, M.~Maasoumy, R.~M. Murray, A.~Sangiovanni-Vincentelli, and
  S.~A. Seshia, ``Model predictive control with signal temporal logic
  specifications,'' in \emph{53rd IEEE Conference on Decision and Control}, Dec
  2014, pp. 81--87.

\bibitem{milp}
S.~Saha and A.~A. Julius, ``{An MILP approach for real-time optimal controller
  synthesis with metric temporal logic specifications},'' in \emph{American
  Control Conference (ACC)}.\hskip 1em plus 0.5em minus 0.4em\relax IEEE, 2016,
  pp. 1105--1110.

\bibitem{husam}
Y.~V. Pant, H.~Abbas, and R.~Mangharam, ``Smooth operator: Control using the
  smooth robustness of temporal logic,'' in \emph{IEEE Conference on Control
  Technology and Applications (CCTA)}, 2017, pp. 1235--1240.

\bibitem{li}
X.~Li, Y.~Ma, and C.~Belta, ``A policy search method for temporal logic
  specified reinforcement learning tasks,'' in \emph{Annual American Control
  Conference (ACC)}.\hskip 1em plus 0.5em minus 0.4em\relax IEEE, 2018, pp.
  240--245.

\bibitem{akazaki}
T.~Akazaki and I.~Hasuo, ``Time robustness in mtl and expressivity in hybrid
  system falsification,'' in \emph{International Conference on Computer Aided
  Verification}.\hskip 1em plus 0.5em minus 0.4em\relax Springer, 2015, pp.
  356--374.

\bibitem{filter}
A.~Rodionova, E.~Bartocci, D.~Nickovic, and R.~Grosu, ``{Temporal Logic as
  Filtering},'' in \emph{19th International Conference on Hybrid Systems:
  Computation and Control}.\hskip 1em plus 0.5em minus 0.4em\relax ACM, 2016,
  pp. 11--20.

\bibitem{discrete}
L.~Lindemann and D.~V. Dimarogonas, ``Robust control for signal temporal logic
  specifications using discrete average space robustness,'' \emph{Automatica},
  vol. 101, pp. 377--387, 2019.

\bibitem{GAbook}
D.~P. Bertsekas, \emph{Nonlinear programming}.\hskip 1em plus 0.5em minus
  0.4em\relax Athena scientific Belmont, 1999.

\end{thebibliography}
\end{document}